\documentclass[11pt,a4paper]{article}
\usepackage[hyperref]{emnlp-ijcnlp-2019}
\usepackage{times}
\usepackage{graphics}
\usepackage{graphicx}
\usepackage{latexsym}
 \usepackage{tikz}
\usepackage{amsthm}
\usepackage{amsmath}
\usepackage{amssymb}
\usepackage{microtype}
\usepackage{adjustbox}
\usepackage{array}
\usepackage{multirow}
\usepackage{url}
\usepackage{color, colortbl}
\usepackage{float}
\usepackage[disable]{todonotes}

\aclfinalcopy 
\newtheorem{prop}{Proposition}

\usepackage{booktabs}
\usepackage{cleveref}
\crefname{section}{\S}{\S\S}
\Crefname{section}{\S}{\S\S}
\crefname{table}{Tab.}{}
\crefname{figure}{Fig.}{}
\crefname{algorithm}{Algorithm}{}
\crefname{equation}{eq.}{}
\crefname{appendix}{App.}{}
\crefname{prop}{Proposition}{}
\crefformat{section}{\S#2#1#3}  

\makeatletter

\newcommand*\iftodonotes{\if@todonotes@disabled\expandafter\@secondoftwo\else\expandafter\@firstoftwo\fi}  
\makeatother

\newcommand{\note}[4][]{\todo[author=#2,color=#3,size=\scriptsize,fancyline,caption={},#1]{#4}} 

\newcommand{\rowan}[2][]{\note[#1]{rowan}{green!40}{#2}}
\newcommand{\simone}[2][]{\note[#1]{simone}{blue!40}{#2}}

\newcommand{\defn}[1]{\textbf{#1}}

\newcommand{\vv}{\mathbf{v}}

\newcommand{\word}[1]{\textit{#1}}

\definecolor{LightGrey}{gray}{0.9}
\newcommand{\citeposs}[1]{\citeauthor{#1}'s (\citeyear{#1})}

\newcolumntype{R}[2]{%
    >{\adjustbox{angle=#1,lap=\width-(#2)}\bgroup}%
    l%
    <{\egroup}%
}

\DeclareCaptionLabelFormat{andtable}{#1~#2  \&  \tablename~\thetable}

\title{\raisebox{1ex}[0in][0in]{\parbox[b]{\linewidth}{\begin{flushright}\footnotesize
        \textmd{\textsf{\textcolor{gray}{Appeared in the
                proceedings of EMNLP--IJCNLP 2019 (Hong Kong, November). \\ This
          clarified version was prepared in December 2019.}}}\end{flushright}}}\\ \vspace{-1ex}It's All in the Name: Mitigating Gender Bias with Name-Based Counterfactual Data Substitution}

\author{Rowan Hall Maudslay$^1$ \hspace{1.5em} Hila Gonen$^2$ \hspace{1.5em} Ryan Cotterell$^1$ \hspace{1.5em} Simone Teufel$^1$ \\
  ${}^1$ Department of Computer Science and Technology, University of Cambridge \\ ${}^2$ Department of Computer Science, Bar-Ilan University \\
  {\tt \{rh635,rdc42,sht25\}@cam.ac.uk} \; {\tt hilagnn@gmail.com} \\}
  
\date{}

\begin{document}
\maketitle
\begin{abstract}
This paper treats gender bias latent in word embeddings. Previous mitigation attempts rely on the operationalisation of gender bias as a projection over a linear subspace. An alternative approach is \emph{Counterfactual Data Augmentation} (CDA), in which a corpus is duplicated and augmented to remove bias, e.g. by swapping all inherently-gendered words in the copy. We perform an empirical comparison of these approaches on the English Gigaword and Wikipedia, and find that whilst both successfully reduce direct bias and perform well in tasks which quantify embedding quality, CDA variants outperform projection-based methods at the task of drawing non-biased gender analogies by an average of 19\% across both corpora.
We propose two improvements to CDA: \emph{Counterfactual Data Substitution} (CDS), a variant of CDA in which potentially biased text is randomly substituted to avoid duplication, and the Names Intervention, a novel name-pairing technique that vastly increases the number of words being treated. CDA/S with the Names Intervention is the only approach which is able to mitigate indirect gender bias: following debiasing, previously biased words are significantly less clustered according to gender (cluster purity is reduced by 49\%), thus improving on the state-of-the-art for bias mitigation.
\simone{I commented the last sentence out. Abstract is very long.}\rowan{there was some one-line weirdness going on later so I added a part of it back in, that alright?}
\end{abstract}
 
\section{Introduction}

\emph{Gender bias} describes an inherent prejudice against a gender, captured both by individuals and larger social systems. 
Word embeddings, a popular machine-learnt semantic space, have been shown to 
retain gender bias present in corpora used to train them \citep{caliskan}.
This results in gender-stereotypical vector analogies {\`a} la \newcite{NIPS2013_5021}, such as \word{man}:\word{computer~programmer} :: \word{woman}:\word{homemaker} \citep{DBLP:conf/nips/BolukbasiCZSK16}, and such bias has been shown to materialise in a variety of downstream tasks, e.g. coreference resolution \citep{CorefRes, CorefRes2}.

By operationalising gender bias in word embeddings as a linear subspace, \newcite{DBLP:conf/nips/BolukbasiCZSK16} are able to debias with simple techniques from linear algebra. Their method successfully mitigates
\simone{does not particularly like boldfacing for emphasis, but can live with.}\defn{direct bias}: \word{man} is no longer more similar to \word{computer~programmer} in vector space than \word{woman}. However, the structure of gender bias in vector space remains largely intact, and the new vectors still evince \defn{indirect bias}: associations which result from gender bias between not explicitly gendered words, for example a possible association between \word{football} and \word{business} resulting from their mutual association with explicitly masculine words \citep{hila}. In this paper we continue the work of \citeauthor{hila}, and
show that another paradigm for gender bias mitigation proposed by \citet{lu}, Counterfactual Data Augmentation (CDA), is also unable to mitigate indirect bias. We also show, using a new test we describe (non-biased gender analogies), that WED might be removing too much gender information, casting further doubt on its operationalisation of gender bias as a linear subspace.  


To improve CDA we make two proposals. The first, Counterfactual Data Substitution (CDS), is designed to avoid text duplication in favour of substitution. The second, the Names Intervention, is a method which can be applied to either CDA or CDS, and treats bias inherent in first names. It does so using a novel name pairing strategy that accounts for both name frequency and gender-specificity. Using our improvements, the clusters of the most biased words exhibit a reduction of cluster purity by an average of 49\% across both corpora following treatment, thereby offering a partial solution to the problem of indirect bias as formalised by \citet{hila}. \simone{first part of reaction to reviewer 4}Additionally, although one could expect that the debiased embeddings might suffer performance losses in computational linguistic tasks, our embeddings remain useful for at least two such tasks, word similarity and sentiment classification \cite{doc2vec}. \looseness -1


\section{Related Work}

The measurement and mitigation of gender bias relies on the chosen operationalisation of gender bias. 
As a direct consequence, how researchers 
choose to operationalise bias determines both the techniques at one's disposal to mitigate the bias, as well as the yardstick by which success is determined. 

\subsection{Word Embedding Debiasing} \label{sec:gender-direction}

One popular method for the mitigation of gender bias, introduced by \newcite{DBLP:conf/nips/BolukbasiCZSK16}, measures the genderedness of words by the extent to which they point in a 
\defn{gender direction}. Suppose we embed our words into $\mathbb{R}^d$. The fundamental assumption is that there exists a linear subspace $B \subset \mathbb{R}^d$ that contains (most of) the gender bias in the space of word embeddings. (Note that $B$ is a direction
when it is a single vector.) We term this assumption the \defn{gender subspace hypothesis}.  Thus, by 
basic linear algebra, we may decompose any word vector $\vv \in \mathbb{R}^d$ as
the sum of the projections onto the bias subspace and its complement: $\vv = \vv_{B} + \vv_{\perp B}$. The (implicit) operationalisation of gender bias under this hypothesis is, then, the magnitiude of the bias vector $||\vv_{B}||_2$.

To capture $B$, \citet{DBLP:conf/nips/BolukbasiCZSK16} 
first construct $N$ sets $D_i$, each of which
contains a pair of words that differ in their gender but that are otherwise semantically equivalent (using a predefined set of gender-definitional pairs). For example,
\{\word{man}, \word{woman}\} would be one
set and \{\word{husband}, \word{wife}\} would be another. They then compute the average empirical covariance matrix
\begin{equation}
    C = \sum_{i=1}^N \frac{1}{|D_i|}\sum_{w \in D_i} (\vec{w} - \mu_i)(\vec{w} - \mu_i)^{\top}
\end{equation}
where $\mu_i$ is the mean embedding of the words in $D_i$, then $B$ is taken to be the space spanned by the top $k$ eigenvectors of $C$ associated with the largest eigenvalues. \citeauthor{DBLP:conf/nips/BolukbasiCZSK16} set $k=1$, and thus define a gender direction. 

Using this operalisation of gender bias, \citeauthor{DBLP:conf/nips/BolukbasiCZSK16} go on to provide a linear-algebraic method (Word Embedding Debiasing, WED, originally ``hard debiasing'') to remove gender bias in two phases: first, for non-gendered words, the gender direction is removed (``neutralised''). Second, pairs of gendered words such as \word{mother} and \word{father} are made equidistant to all non-gendered words (``equalised''). Crucially, under the gender subspace hypothesis, it is only necessary to
identify the subspace $B$ as it is possible to perfectly remove the bias
under this operationalisation using tools from numerical linear algebra.



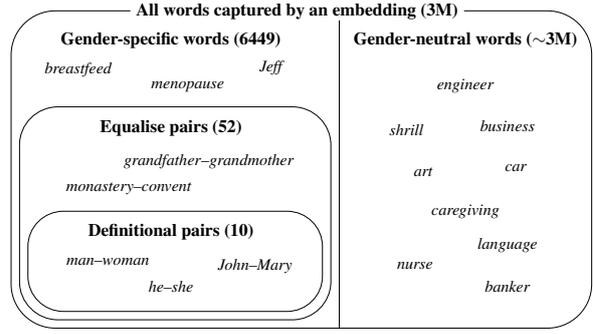
\begin{figure}
\centering
\resizebox{\columnwidth}{!}{%
\begin{tikzpicture}

\draw[fill=blue!0!white, draw=black, rounded corners=30pt] (0.2,0.2) rectangle (14,7.8);
\draw[draw=black, rounded corners=26pt] (0.4,0.4) rectangle (7.8,5.5);
\draw[draw=black, rounded corners=24pt] (0.6,0.6) rectangle (7.6,3);

\large{
\node[fill=blue!0!white, , rounded corners=5pt] at (7,7.8) {\bf All words captured by an embedding (3M)};
}
\draw[black] (8,0.2) -- (8,7.57);

\large{
\node at (4,7.1) {\bf Gender-specific words (6449)};
\node at (11,7.1) {\bf Gender-neutral words ($\mathbf{\sim}$3M)};
\node at (4,5) {\bf Equalise pairs (52)};
\node at (4,2.5) {\bf Definitional pairs (10)};
}

\normalsize{
\node at (1.8,6.39) {\word{breastfeed}};
\node at (6.4,6.45) {\word{Jeff}};
\node at (4.4,6) {\word{menopause}};

\node at (3,3.55) {\word{monastery}--\word{convent}};
\node at (4.9,4.2) {\word{grandfather}--\word{grandmother}};

\node at (2.5,1.8) {\word{man}--\word{woman}};
\node at (4,1.2) {\word{he}--\word{she}};
\node at (6,1.7) {\word{John}--\word{Mary}};

\node at (9.8,1.7) {\word{nurse}};
\node at (12,1.2) {\word{banker}};
\node at (12,2.2) {\word{language}};
\node at (12,5.05) {\word{business}};
\node at (9.6,4.95) {\word{shrill}};
\node at (10,3.95) {\word{art}};
\node at (11,3) {\word{caregiving}};
\node at (12.2,4.05) {\word{car}};
\node at (11,6) {\word{engineer}};

}
\end{tikzpicture}
}
\caption{Word sets used by WED with examples}
\label{fig:word_sets}
\end{figure}

The method uses three sets of words or word pairs: 10 definitional pairs (used to define the gender direction), 218 gender-specific seed words (expanded to a larger set using a linear classifier, the compliment of which is neutralised in the first step), and 52 equalise pairs (equalised in the second step). The relationships among these sets are illustrated in Figure~\ref{fig:word_sets}; for instance, gender-neutral words are defined as all words in an embedding that are not gender-specific.   

\citeauthor{DBLP:conf/nips/BolukbasiCZSK16} find that this method results in a 68\% reduction of stereotypical analogies as identified by human judges. However, bias is removed only insofar as the operationalisation allows. In a comprehensive analysis, \newcite{hila} show that the original structure of bias in the WED embedding space remains intact. 

\subsection{Counterfactual Data Augmentation}  \label{cda_initial}

As an alternative to WED, \citet{lu} propose Counterfactual Data Augmentation (CDA), in which a text transformation designed to invert bias is performed on a text corpus, the result of which is then appended to the original, to form a new bias-mitigated corpus used for training embeddings. 
Several interventions are proposed: in the simplest, occurrences of words in 124 gendered word pairs are swapped. For example, `the woman cleaned the kitchen' would (counterfactually) become `the man cleaned the kitchen' as \word{man}--\word{woman} is on the list. Both versions would then together be used in embedding training, in effect neutralising the \word{man}--\word{woman} bias.

The \emph{grammar intervention}, \citeauthor{lu}'s improved intervention, uses coreference information to veto swapping gender words when they corefer to a proper noun.\footnote{We interpret \citeposs{lu} phrase ``cluster'' to mean ``coreference chain''.} This avoids \word{Elizabeth \dots she \dots queen} being changed to, for instance, \word{Elizabeth \dots he \dots king}. It also uses POS information to avoid ungrammaticality related to the ambiguity of \word{her} between personal pronoun and possessive determiner. In the context, `her teacher was proud of her', this results in the correct sentence `\emph{his} teacher was proud of \emph{him}'. 

\section{Improvements to CDA}

We prefer the philosophy of CDA over WED as it makes fewer assumptions about the operationalisation of the bias it is meant to mitigate.

\subsection{Counterfactual Data Substitution}

The duplication of text which lies at the heart of CDA will produce debiased corpora with peculiar statistical properties unlike those of naturally occurring text. Almost all observed word frequencies will be even, with a notable jump from 2 directly to 0, and a  type--token ratio far lower than predicted by Heaps' Law for text of this length. The precise effect this will have on the resulting embedding space is hard to predict, 
but we assume that it is preferable not to violate the fundamental assumptions of the algorithms used to create embeddings. 

As such, we propose to apply substitutions probabilistically (with 0.5 probability), which results in a non-duplicated counterfactual training corpus, a method we call \textbf{Counterfactual Data Substitution (CDS)}. Substitutions are performed on a per-document basis in order to maintain grammaticality and discourse coherence. 
This simple change should have advantages in terms of naturalness of text and processing efficiency, as well as theoretical foundation. 

\subsection{The Names Intervention} \label{sec:names}

Our main technical contribution in this paper is to provide a method for better counterfactual augmentation, which is based on bipartite-graph matching of names. Instead of Lu et. al's (2018) \nocite{lu} solution of not treating words which corefer to proper nouns in order to maintain grammaticality, we propose an explicit treatment of first names. This is because we note that as a result of not swapping the gender of words which corefer with proper nouns, CDA could in fact reinforce certain biases instead of mitigate them. Consider the sentence `Tom \dots He is a successful and powerful executive'. Since \word{he} and  \word{Tom} corefer, the counterfactual corpus copy will not replace \word{he} with \word{she} in this instance, and as the method involves a duplication of text, this would result in a stronger, not weaker, association between \word{he} and gender-stereotypic concepts present like \word{executive}. Even under CDS, this would still mean that biased associations are left untreated (albeit at least not reinforced). Treating names should in contrast effect a real neutralisation of bias, with the added bonus that grammaticality is maintained without the need for coreference resolution.
 
 \begin{figure}
\center{\includegraphics[width=\linewidth,trim={0 0 1.1cm 1.3cm},clip]
        {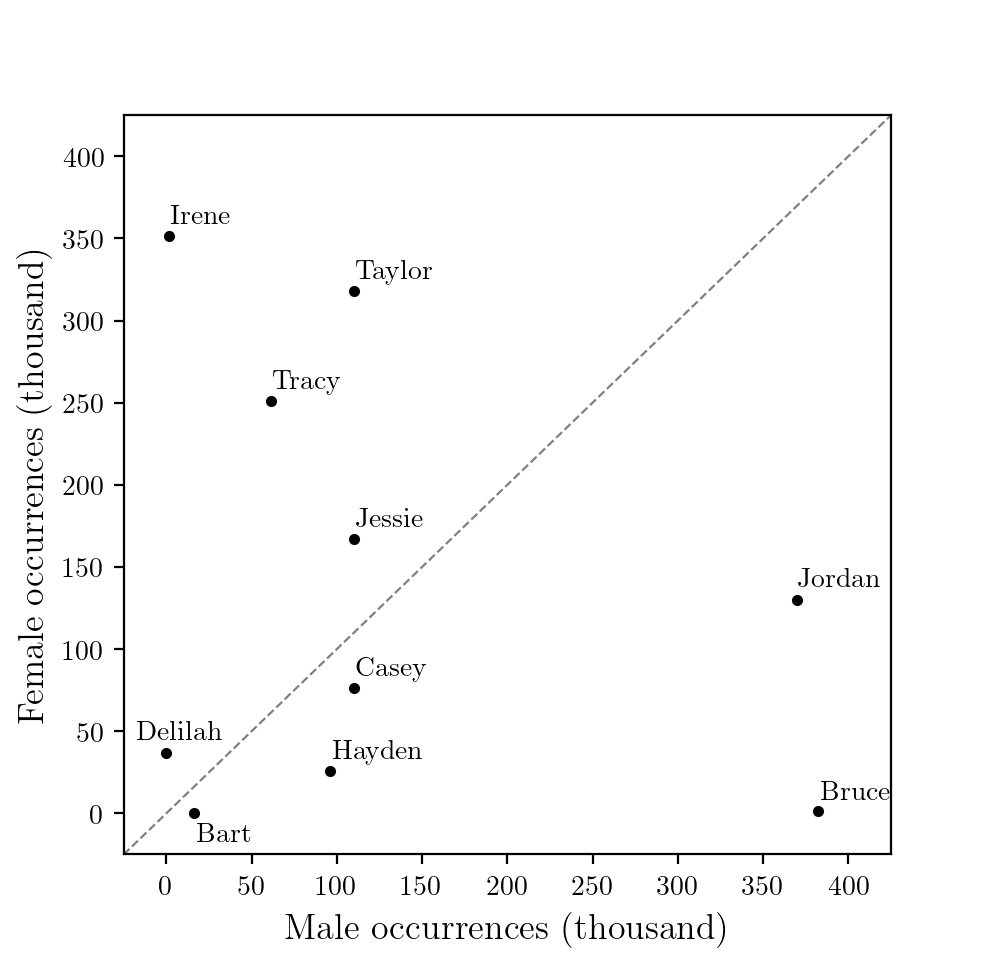}}
\caption{Frequency and gender-specificity of names in the SSA dataset}
 \label{fig:names_1}
\end{figure}

The United States Social Security Administration (SSA) 
dataset contains a list of all first names from Social Security card applications for births in the United States after 1879, along with their gender.\footnote{\url{https://www.ssa.gov/oact/babynames/background.html}}
Figure~\ref{fig:names_1} plots a few example names according to their male and female occurrences, and shows that names have varying degrees of gender-specificity.\footnote{The dotted line represents gender-neutrality, and more frequent names are located further from the origin.}

We fixedly associate pairs of names for swapping, thus expanding \citeauthor{lu}'s short list of gender pairs vastly. Clearly both name frequency and the degree of gender-specificity are relevant to this bipartite matching. If only frequency were considered, a more gender-neutral name (e.g. \word{Taylor}) could be paired with a very gender-specific name (e.g. \word{John}), which would negate the gender intervention in many cases (namely whenever a male occurrence of \word{Taylor} is transformed into \word{John}, which would also result in incorrect pronouns, if present). If, on the other hand, only the degree of gender-specificity were considered, we would see frequent names (like \word{James}) being paired with far less frequent names (like \word{Sybil}), which would distort the overall frequency distribution of names. This might also result in the retention of a gender signal: for instance, swapping a highly frequent male name with a rare female name might simply make the rare female name behave as a new link between masculine contexts (instead of the original male name), as it rarely appears in female contexts.


 

Figure~\ref{fig:names_2} shows a plot of various names' number of primary gender\footnote{Defined as its most frequently occurring gender.} occurances against their secondary gender occurrences, with red dots for primary-male and blue crosses for primary-female names.\footnote{The hatched area demarcates an area of the graph where no names can exist: if any name did then its primary and secondary gender would be reversed and it would belong to the alternate set.} 
The problem of finding name-pairs thus decomposes into a Euclidean-distance bipartite matching problem, which can be solved using the Hungarian method \citep{kuhn}. We compute pairs for the most frequent 2500 names of each gender in the SSA dataset. There is also the problem that many names are also common nouns (e.g. \word{Amber}, \word{Rose}, or \word{Mark}), which we solve using Named Entity Recognition. 

\begin{figure}
\center{\includegraphics[width=\linewidth,trim={0 0 1.1cm 1.3cm},clip]
        {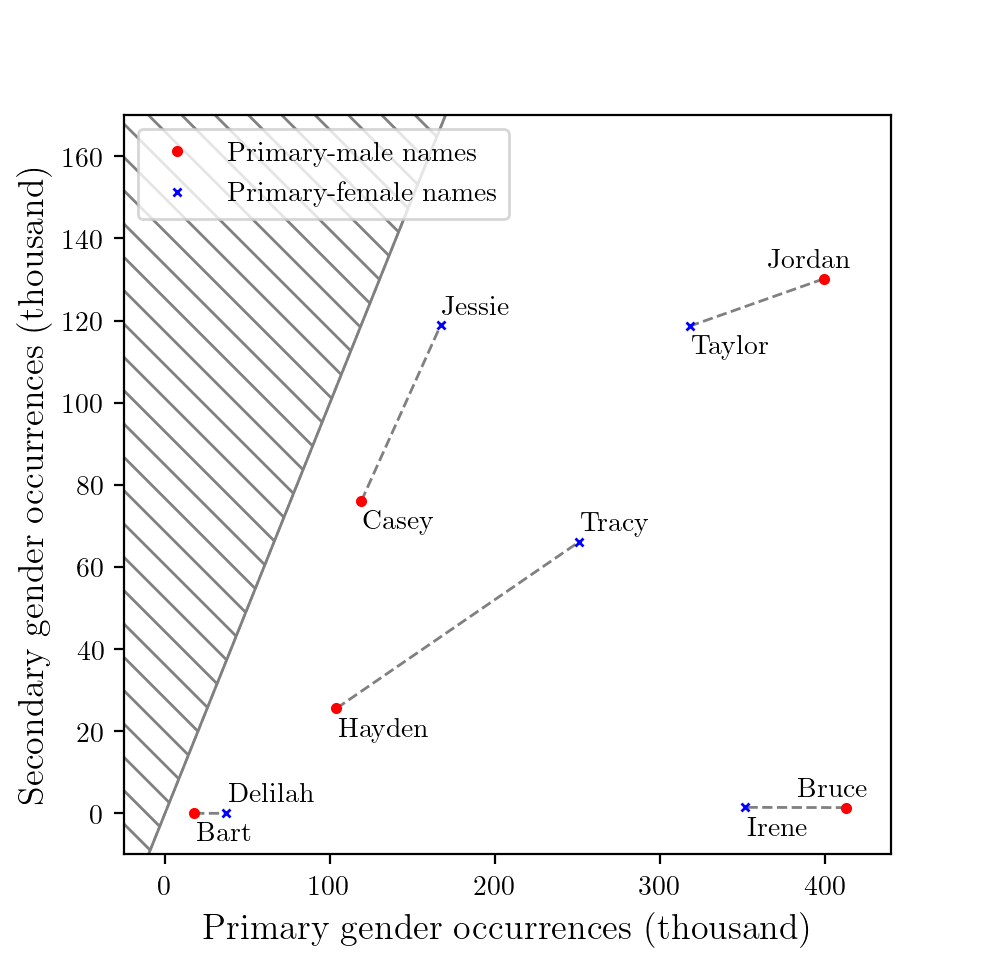}}
\caption{Bipartite matching of names by frequency and gender-specificity}
 \label{fig:names_2}
\end{figure}

\section{Experimental Setup}

We compare eight variations of the mitigation methods. {\bf CDA} is our reimplementation of \citeposs{lu} na\"{i}ve intervention, \textbf{gCDA} uses their grammar intervention, and \textbf{nCDA} uses our new Names Intervention. {\bf gCDS} and \textbf{nCDS} are variants of the grammar and Names Intervention using CDS. {\bf WED40} is our reimplementation of \citeposs{DBLP:conf/nips/BolukbasiCZSK16} method, which (like the original) uses a single component to define the gender subspace, accounting for $>40\%$ of variance. As this is much lower than in the original paper (where it was 60\%, reproduced in Figure~\ref{fig:pca_1}), we define a second space, {\bf WED70}, which uses a 2D subspace accounting for $>70\%$ of variance. To test whether WED profits from additional names, we use the 5000 paired names in the names gazetteer as additional equalise pairs (\textbf{nWED70}).\footnote{We use the 70\% variant as preliminary experimentation showed that it was superior to WED40.} As control, we also evaluate the unmitigated space ({\bf none}).\looseness=-1

We perform an empirical comparison of these bias mitigation techniques on two corpora, the Annotated English Gigaword \citep{gigaword} and Wikipedia. Wikipedia is of particular interest, since though its Neutral Point of View (NPOV) policy\footnote{\url{https://en.wikipedia.org/wiki/Wikipedia:Neutral_point_of_view}} predicates that all content should be presented without bias, women are nonetheless less likely to be deemed ``notable'' than men of equal stature \citep{wikibrit}, and there are differences in the choice of language used to describe them \citep{bamman, firstwomen}. We use the annotation native to the Annotated English Gigaword, and process Wikipedia with CoreNLP (statistical coreference; bidirectional tagger). Embeddings are created using Word2Vec\footnote{A CBOW model was trained over five epochs to produce 300 dimensional embeddings. Words were lowercased, punctuation other than underscores and hyphens removed, and tokens with fewer than ten occurrences were discarded.}. We use the original complex lexical input (gender-word pairs and the like) for each algorithm as we assume that this benefits each algorithm most. \simone{I am not 100\% sure of which "expansion" you are talking about here. The classifier Bolucbasi use maybe?}\rowan{yup - clarified} Expanding the set of gender-specific words for WED (following \citeauthor{DBLP:conf/nips/BolukbasiCZSK16}, using a linear classifier) on Gigaword resulted in 2141 such words, 7146 for Wikipedia.\footnote{We modify or remove some phrases from the training data not included in the vocabulary of our embeddings.}

\begin{figure}
\center{\includegraphics[width=\linewidth,trim=1.7cm 0cm 1.9cm 1.3cm,clip]
        {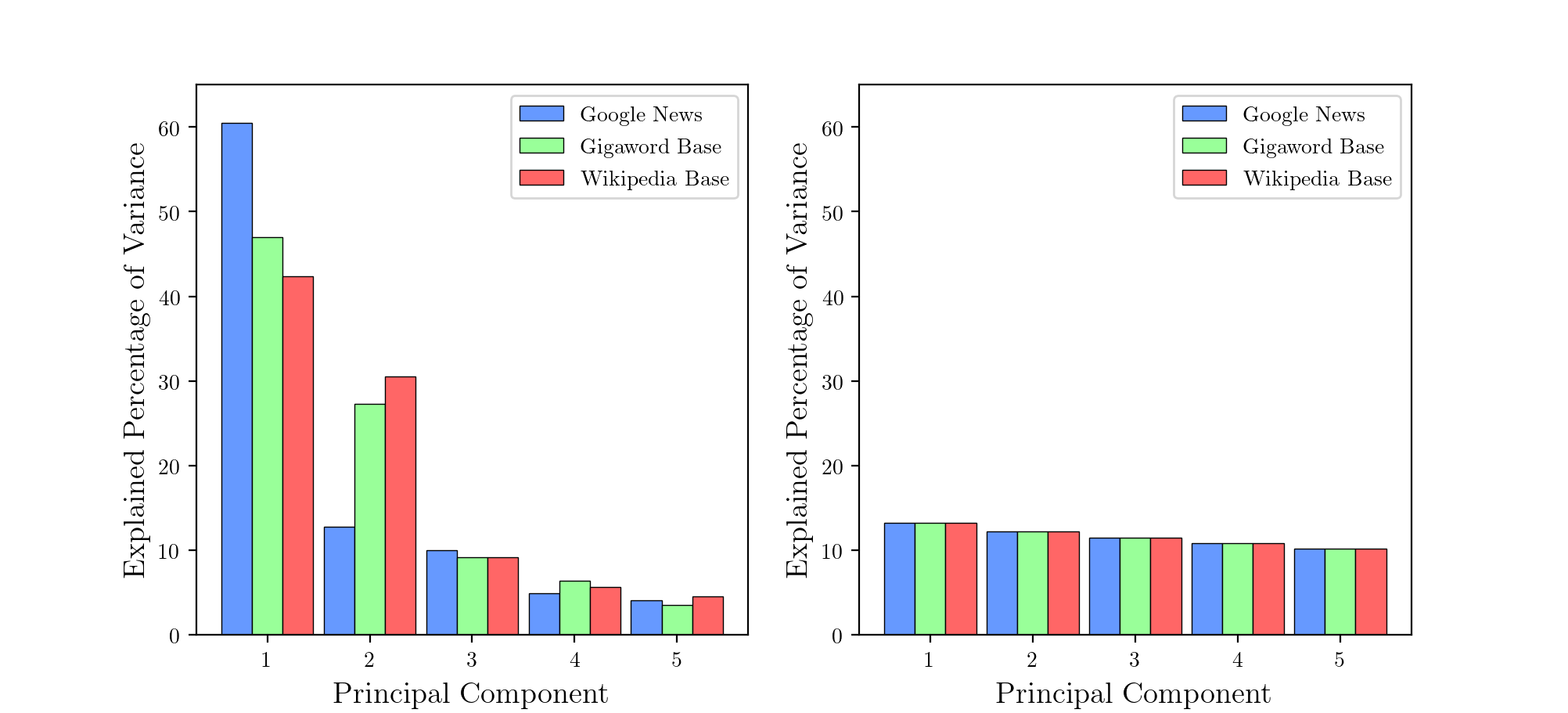}}
\caption{Variance explained by the top Principal Components of the definitional word pairs (left) and random unit vectors (right)}
\label{fig:pca_1}
\end{figure}

In our experiments, we test the degree to which the spaces are successful at mitigating direct and indirect bias, as well as the degree to which they can still be used in two NLP tasks standardly performed with embeddings, word similarity and sentiment classification. We also introduce one further, novel task, which is designed to quantify how well the embedding spaces capture an understanding of gender using non-biased analogies. Our evaluation matrix and methodology is expanded below. \looseness=-1

\paragraph{Direct bias} \citet{caliskan} introduce the Word Embedding Association Test (WEAT), which provides results analogous to earlier psychological work by \citet{greenwald} by measuring the difference in relative similarity between two sets of target words $X$ and $Y$ and two sets of attribute words $A$ and $B$. We compute Cohen's $d$ (a measure of the difference in relative similarity of the word sets within each embedding; higher is more biased), and a one-sided $p$-value which indicates whether the bias detected by WEAT within each embedding is significant (the best outcome being that no such bias is detectable). We do this for three tests proposed by \citet{nosek} which measure the strength of various gender stereotypes: art--maths, arts--sciences, and careers--family.\footnote{In the careers--family test the gender dimension is expressed by female and male first names, unlike in the other sets, where pronouns and typical gendered words are used.} 

\paragraph{Indirect bias} To demonstrate indirect gender bias we adapt a pair of methods proposed by \citet{hila}. First, we test whether the most-biased words prior to bias mitigation remain clustered following bias mitigation. To do this, we define a new subspace, $\vec{b}_\text{test}$, using the 23 word pairs used in the Google Analogy family test subset \citep{NIPS2013_5021} following \citeposs{DBLP:conf/nips/BolukbasiCZSK16} method, and determine the 1000 most biased words in each corpus
(the 500 words most similar to $\vec{b}_\text{test}$ and  $-\vec{b}_\text{test}$) in the unmitigated embedding. For each debiased embedding we then project these words into 2D space with tSNE \citep{tSNE}, compute clusters with \emph{k}-means, and calculate the clusters' V-measure \cite{vmeasure}. Low values of cluster purity indicate that biased words are less clustered following bias mitigation. 

Second, we test whether a classifier can be trained to reclassify the gender of debiased words. If it succeeds, this would indicate that bias-information still remains in the embedding. We trained an RBF-kernel SVM classifier on a random sample of 1000 out of the 5000 most biased words from each corpus using $\vec{b}_\text{test}$ (500 from each gender), then report the classifier's accuracy when reclassifying the remaining 4000 words. 

\paragraph{Word similarity} The quality of a space is traditionally measured by how well it replicates human judgements of word similarity. 
The SimLex-999 dataset \citep{simlex} provides a ground-truth measure of similarity produced by 500 native English speakers.\footnote{It explicitly quantifies similarity rather than association or relatedness; pairs of entities like \word{coffee} and \word{cup} have a low rating.} Similarity scores in an embedding are computed as the cosine angle between word-vector pairs, and  Spearman correlation between  embedding and human judgements are reported. We measure correlative significance at $\alpha = 0.01$.

\paragraph{Sentiment classification} 
Following \citet{doc2vec}, 
we use a standard sentiment classification task to quantify the downstream performance of the embedding spaces when they are used as a pretrained word embedding input \citep{w2vd2v} to Doc2Vec on the Stanford Large Movie Review dataset. 
The classification is performed by an SVM classifier using the document embeddings as features, trained on 40,000 labelled reviews and tested on the remaining 10,000 documents, reported as error percentage. 

\paragraph{Non-biased gender analogies} 
When prop\-os\-ing WED,
\citet{DBLP:conf/nips/BolukbasiCZSK16} use human raters to class gender-analogies as either \emph{biased} (\word{woman}:\word{housewife} :: \word{man}:\word{shopkeeper}) or \emph{appro\-priate} (\word{woman}:\word{grandmother} :: \word{man}::\word{grandfather}), 
and postulate that whilst biased analogies are un\-desirable, appropriate ones should remain.
Our new analogy test uses the 506 analogies in the \emph{fam\-ily analogy} subset of the Google Analogy Test set \citep{NIPS2013_5021} to define many such appropriate analogies that should hold even in a debiased environment, such as \word{boy}:\word{girl} :: \word{nephew}:\word{niece}.\footnote{The entire Google Analogy Test set contains 19,544 analogies, which are usually reported as a single result or as a pair of semantic and syntactic results.}
We use a proportional pair-based analogy test, which measures each embedding's performance when drawing a fourth word to complete each analogy, and report error percentage. \looseness -1

\section{Results}

\begin{table}
\setlength\tabcolsep{1.5pt}
\centering
\resizebox{\columnwidth}{!}{
{\small
\begin{tabular}[b]{lcccccc}
\toprule
&\multicolumn{2}{c}{Art--Maths}
&\multicolumn{2}{c}{Arts--Sciences}
&\multicolumn{2}{c}{Career--Family} \\
Method & $d$ & $p$ & $d$ & $p$ & $d$ & $p$  \\
\midrule
& \multicolumn{6}{c}{\bf Gigaword}\\
\rowcolor{LightGrey} none & $1.32$ & $<10^{-2}$ & $1.50$ & $<10^{-3}$ & $1.74$& $<10^{-4}$ \\
CDA & $0.67$ & $.10$ & $1.05$ & $.02$ & $1.79$ & $<10^{-4}$ \\
gCDA & $1.16$ & $.01$ & $1.46$ & $<10^{-2}$ & $1.77$ & $<10^{-4}$ \\
nCDA & $-0.49$\phantom{$-$} & $.83$ & $0.34$ & $.27$ & $1.45$ & $<10^{-3}$ \\
\rowcolor{LightGrey} gCDS & $0.96$ & $.03$ & $1.31$ & $<10^{-2}$ & $1.78$ & $<10^{-4}$ \\
\rowcolor{LightGrey} nCDS & $-0.19$\phantom{$-$} & $.63$ & $0.48$ & $.19$ & $1.45$ & $<10^{-3}$ \\
WED40 & $-0.73$\phantom{$-$} & $.92$ & $0.31$ & $.28$ & $1.24$ & $<10^{-2}$ \\
WED70 & $-0.73$\phantom{$-$} & $.92$ & $0.30$ & $.29$ & $1.15$ & $<10^{-2}$ \\
nWED70 & $0.30$ & $.47$ & $0.54$ & $.19$ & $0.59$ & $.15$ \\
\midrule
& \multicolumn{6}{c}{\bf Wikipedia}\\

\rowcolor{LightGrey} none & $1.64$ & $<10^{-3}$ & $1.51$ & $<10^{-3}$ & $1.88$& $<10^{-4}$ \\
CDA & $1.58$ & $<10^{-3}$ & $1.66$ & $<10^{-4}$ & $1.87$ & $<10^{-4}$ \\
gCDA & $1.52$ & $<10^{-3}$ & $1.57$ & $<10^{-3}$ & $1.84$ & $<10^{-4}$ \\
nCDA & $1.06$ & $.02$ & $1.54$ & $<10^{-4}$ & $1.65$ & $<10^{-4}$ \\
\rowcolor{LightGrey} gCDS & $1.45$ & $<10^{-3}$ & $1.53$ & $<10^{-3}$ & $1.87$ & $<10^{-4}$ \\
\rowcolor{LightGrey} nCDS & $1.05$ & $.02$ & $1.37$ & $<10^{-3}$ & $1.65$ & $<10^{-4}$ \\
WED40 & $1.28$ & $<10^{-2}$ & $1.36$ & $<10^{-2}$ & $1.81$ & $<10^{-4}$ \\
WED70 & $1.05$ & $.02$ & $1.24$ & $<10^{-2}$ & $1.67$ & $<10^{-3}$ \\
nWED70 & $-0.46$\phantom{$-$} & $.52$ & $-0.42$\phantom{$-$} & $.51$ & $0.85$ & $.05$ \\
\midrule
\emph{\citeauthor{nosek}} & $0.82$ & $<10^{-2}$ & $1.47$ & $<10^{-24}$ & $0.72$ & $<10^{-2}$ \\
\bottomrule
\end{tabular}
}}
\caption{Direct bias results}
\label{tab:weat-wiki}
\setlength\tabcolsep{6pt}
\end{table}

\paragraph{Direct bias} Table~\ref{tab:weat-wiki} presents the $d$ scores and WEAT one-tailed $p$-values, which indicate whether the difference in samples means between targets $X$ and $Y$ and attributes $A$ and $B$ is significant. We also compute a two-tailed $p$-value to determine whether the difference between the various sets is significant.\footnote{Throughout this paper, we test significance in the differences between the embeddings with a two-tailed Monte Carlo permutation test at significance interval $\alpha=0.01$ with $r=10,000$ permutations.}

On Wikipedia, nWED70 outperforms every other method ($p<0.01$), and even at $\alpha=0.1$ bias was undetectable. 
In all CDA/S variants, the Names Intervention performs significantly 
better than other intervention strategies (average $d$ for nCDS across all tests 0.95 vs.\ 1.39 for the best non-names CDA/S variants). Excluding the Wikipedia careers--family test (in which the CDA and CDS variants are indistinguishable at $\alpha=0.01$), the CDS variants are numerically better than their CDA counterparts in 80\% of the test cases, although many of these differences are not significant. 

Generally, we notice a trend of WED reducing direct gender bias slightly better than CDA/S. Impressively, WED even successfully reduces bias in the careers--family test, where gender information is captured by names, which were not in WED's gender-equalise word-pair list for treatment. 




\begin{figure}
\centering
\includegraphics[width=0.85\columnwidth,trim=0.2cm 0.5cm 0.3cm 0.3cm,clip]{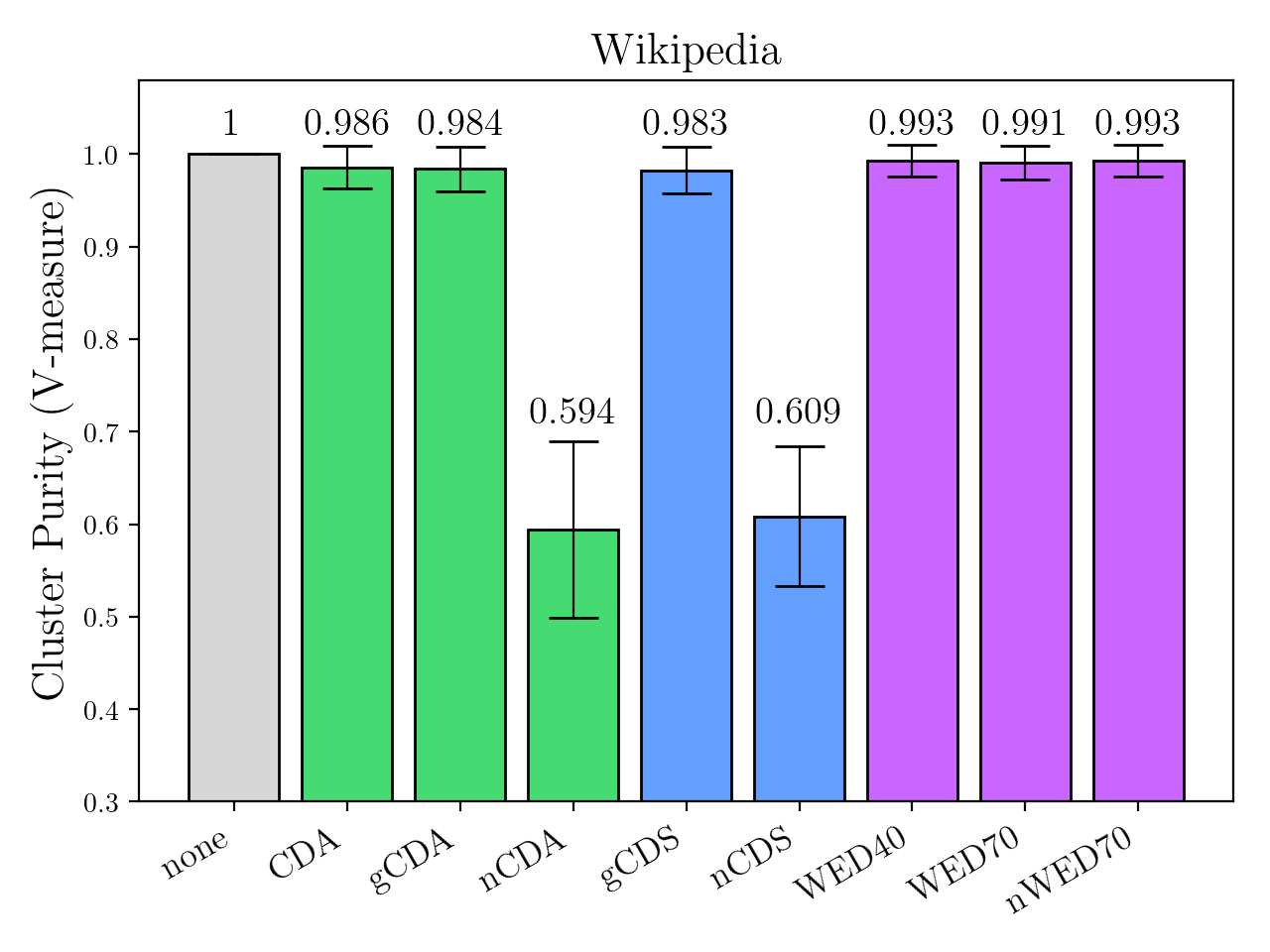}
\caption{Most biased cluster purity results}
\label{fig:purity}
\end{figure}

\begin{figure*}
\center{\includegraphics[width=0.9\linewidth,trim=0.5cm 0.1cm 0.5cm 0.2cm,clip]
        {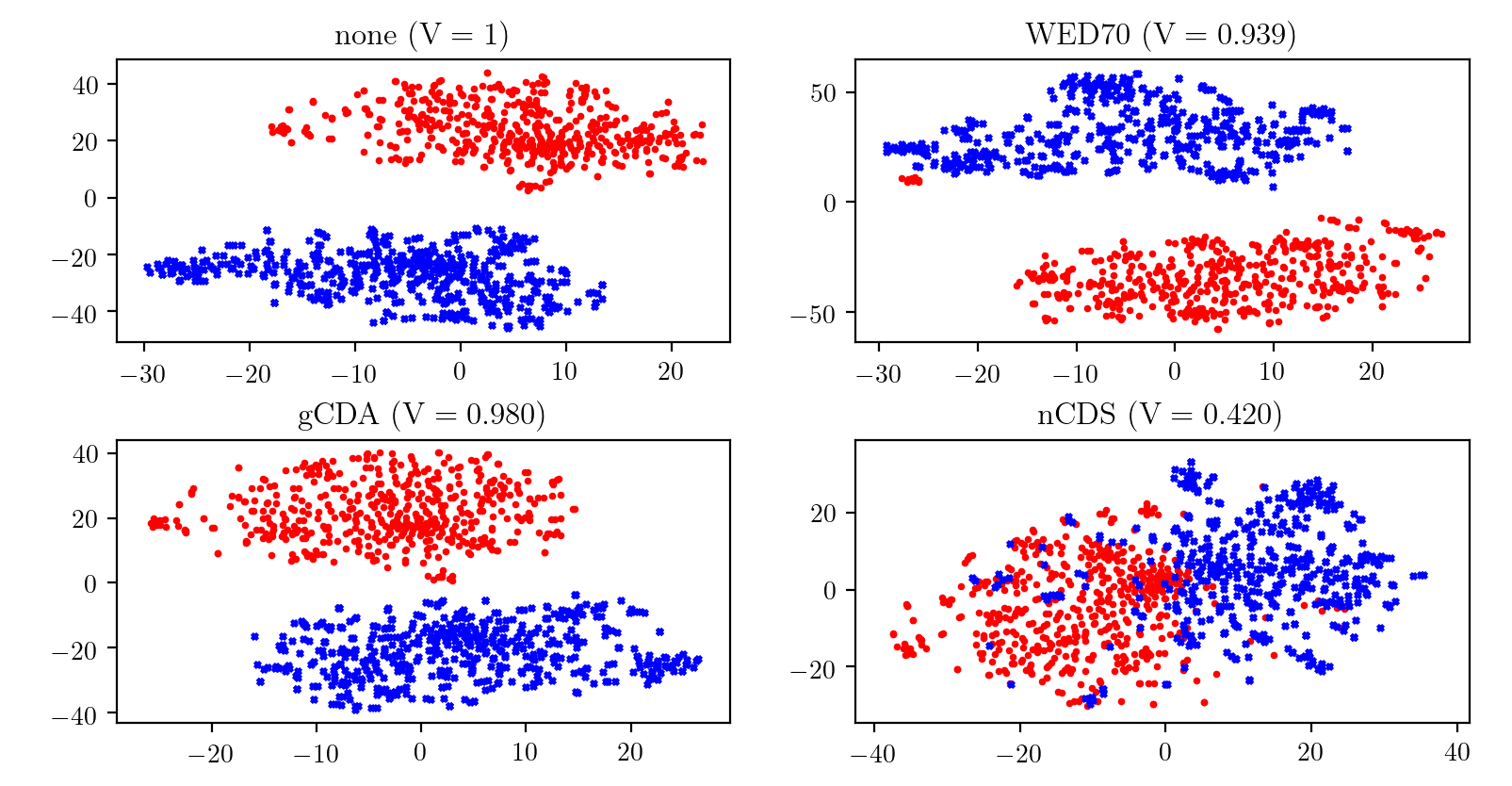}}
\caption{Clustering of biased words (Gigaword)}
 \label{fig:clusters}
\end{figure*}

\paragraph{Indirect bias} Figure~\ref{fig:purity} shows the V-measures of the clusters of the most biased words in Wikipedia for each embedding. Gigaword patterns similarly (see appendix). 
Figure~\ref{fig:clusters} shows example tSNE projections for the Gigaword embeddings (``$\mathrm{V}$'' refers to their V-measures; these examples were chosen as they represent the best results achieved by \citeposs{DBLP:conf/nips/BolukbasiCZSK16} method, \citeposs{lu} method, and our new names variant). 
On both corpora, the new nCDA and nCDS techniques have significantly 
lower purity of biased-word cluster than all other evaluated mitigation techniques (0.420 for nCDS on Gigaword, which corresponds to a reduction of purity by 58\% compared to the unmitigated embedding, and 0.609 (39\%) on Wikipedia). nWED70's V-Measure is significantly higher than either of the other Names variants (reduction of
11\% on Gigaword, only 1\% on Wikipedia), suggesting that the success of nCDS and nCDA is not merely due to their larger list of gender-words. 

\begin{figure}
\centering
\includegraphics[width=0.85\columnwidth,trim=0.2cm 0.5cm 0.3cm 0.3cm,clip]{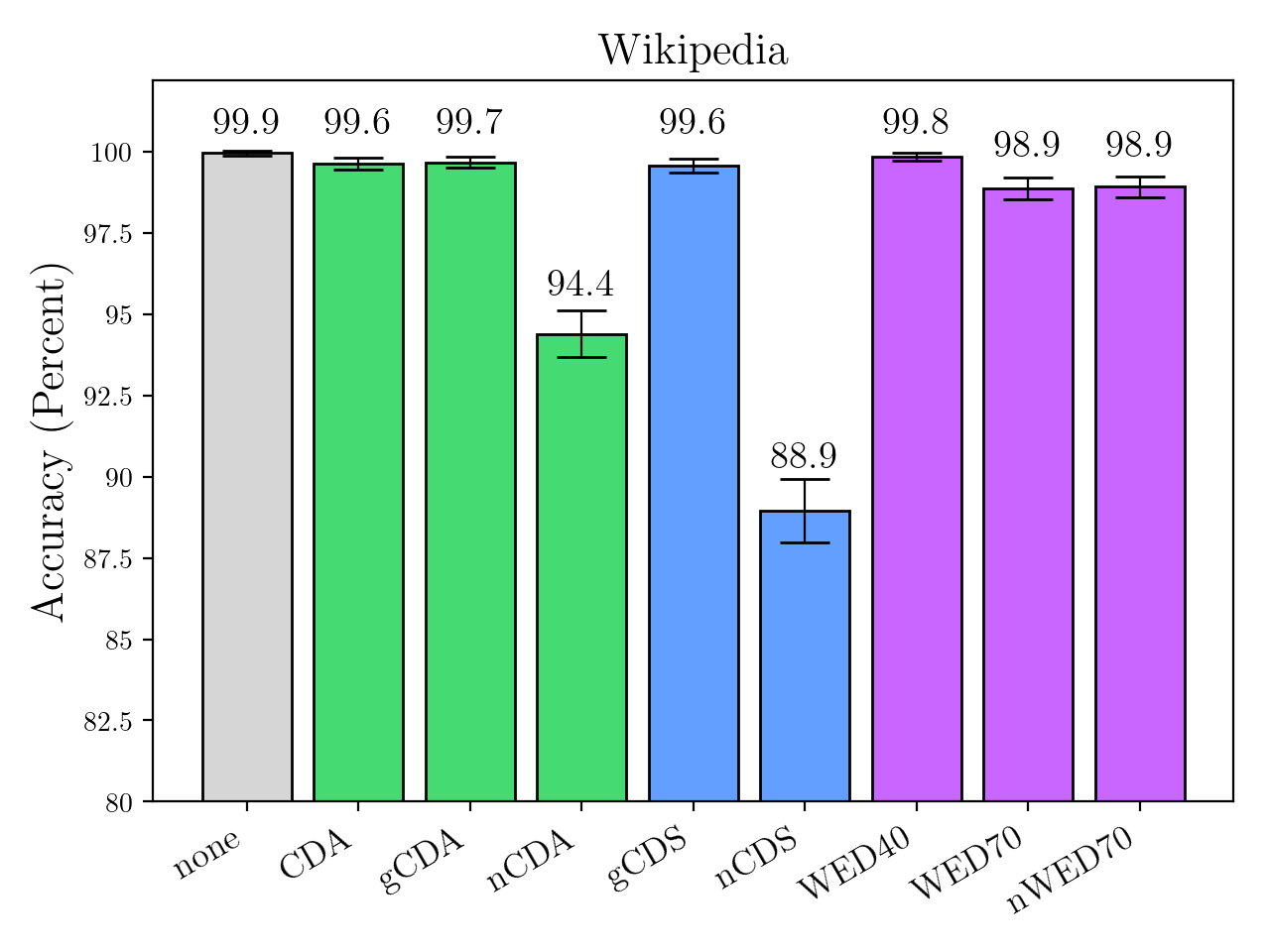}
\caption{Reclassification of most biased words results}
\label{fig:classifier}
\end{figure}

Figure~\ref{fig:classifier} shows the results of the second test of indirect bias, and reports the accuracy of a classifier trained to reclassify previously gender biased words on the Wikipedia embeddings (Gigaword patterns similarly).\footnote{The 95\% confidence interval is calculated by a Wilson score interval, i.e., assuming a normal distribution.} These results reinforce the finding of the clustering experiment: once again, nCDS outperforms all other methods significantly on both corpora ($p<0.01$), although it should be noted that the successful reclassification rate remains relatively high (e.g. 88.9\% on Wikipedia).

We note that nullifying indirect bias associations entirely is not necessarily the goal of debiasing, since some of these may result from causal links in the domain. For example, whilst associations between \word{man} and \word{engineer} and between \word{man} and \word{car} are each stereotypic (and thus could be considered examples of direct bias), an association between \word{engineer} and \word{car} might well have little to do with gender bias, and so should not be mitigated.


\begin{table}
\centering
{\small
\begin{tabular}{l*{2}{c}}
\toprule
Method & \multicolumn{2}{c}{$r_s$} \\
\midrule
& \textbf{Gigaword} & \textbf{Wikipedia} \\ 
\rowcolor{LightGrey} none       		& $0.385$  & $0.368$   \\ 
CDA      	& $0.381$  & $0.363$   \\ 
gCDA   & $0.381$  & $0.363$  \\ 
nCDA 		& $0.380$ & $0.365$  \\ 
\rowcolor{LightGrey} gCDS & $0.382$   & $0.366$   \\ 
\rowcolor{LightGrey} nCDS 	& $0.380$  & $0.362$  \\ 
WED40 & $0.386$  & $0.371$   \\ 
WED70  & $0.395$  & $0.375$  \\ 
nWED70  & $0.384$ & $0.367$  \\ 
\bottomrule
\end{tabular}
}
\caption{Word similarity Results}
\label{tab:simlex}
\end{table}

\paragraph{Word similarity} Table~\ref{tab:simlex} reports the SimLex-999 Spearman rank-order correlation coefficients~$r_s$ (all are significant, $p<0.01$). Surprisingly, the WED40 and 70 methods outperform the unmitigated embedding, although the difference in result is small (0.386 and 0.395 vs.\ 0.385 on Gigaword, 0.371 and 0.367 vs.\ 0.368 on Wikipedia). nWED70, on the other hand, performs worse than the unmitigated embedding (0.384 vs.\ 0.385 on Gigaword, 0.367 vs.\ 0.368 on Wikipedia). CDA and CDS methods do not match the quality of the unmitigated space, but once again the difference is small. \simone{Second Part of Reaction to Reviewer 4.}It should be noted that since SimLex-999 was produced by human raters, it will reflect the human biases these methods were designed to remove, so worse performance might result from successful bias mitigation.

\begin{figure}
\centering
\includegraphics[width=0.85\columnwidth,trim=0.2cm 0.5cm 0.3cm 0.3cm,clip]{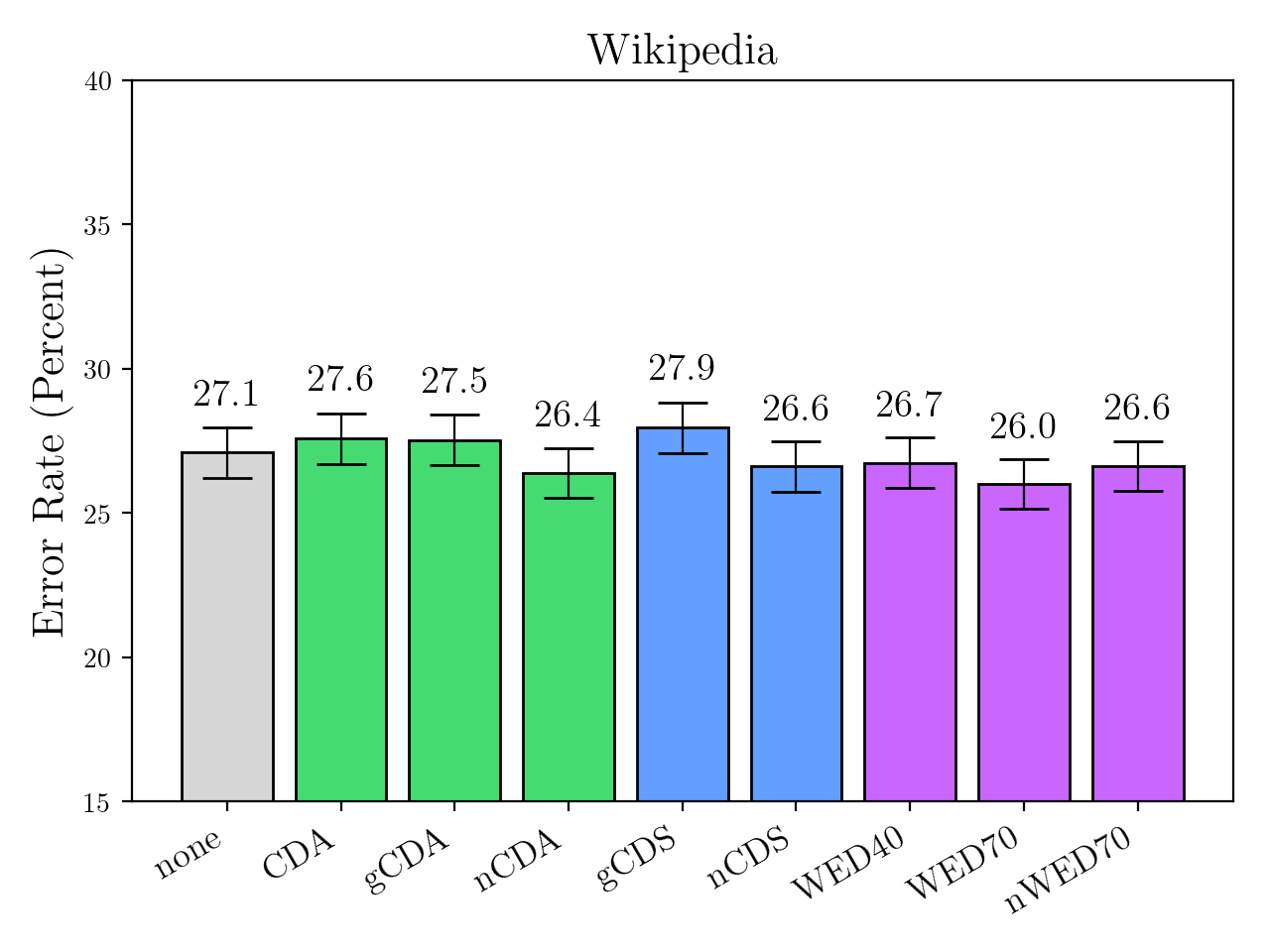}
\caption{Sentiment classification results}
\label{fig:sentiment}
\end{figure}

\paragraph{Sentiment classification} Figure~\ref{fig:sentiment} shows the sentiment classification error rates for Wikipedia (Gigaword patterns similarly).
Results are somewhat inconclusive. While WED70 significantly improves the performance of the sentiment classifier from the unmitigated embedding on both corpora ($p<0.05$), 
the improvement is small (never more than 1.1\%). On both corpora, nothing outperforms WED70 or the Names Intervention variants. \looseness -1

\begin{figure}
\centering
\includegraphics[width=0.85\columnwidth,trim=0.2cm 0.5cm 0.3cm 0.3cm,clip]{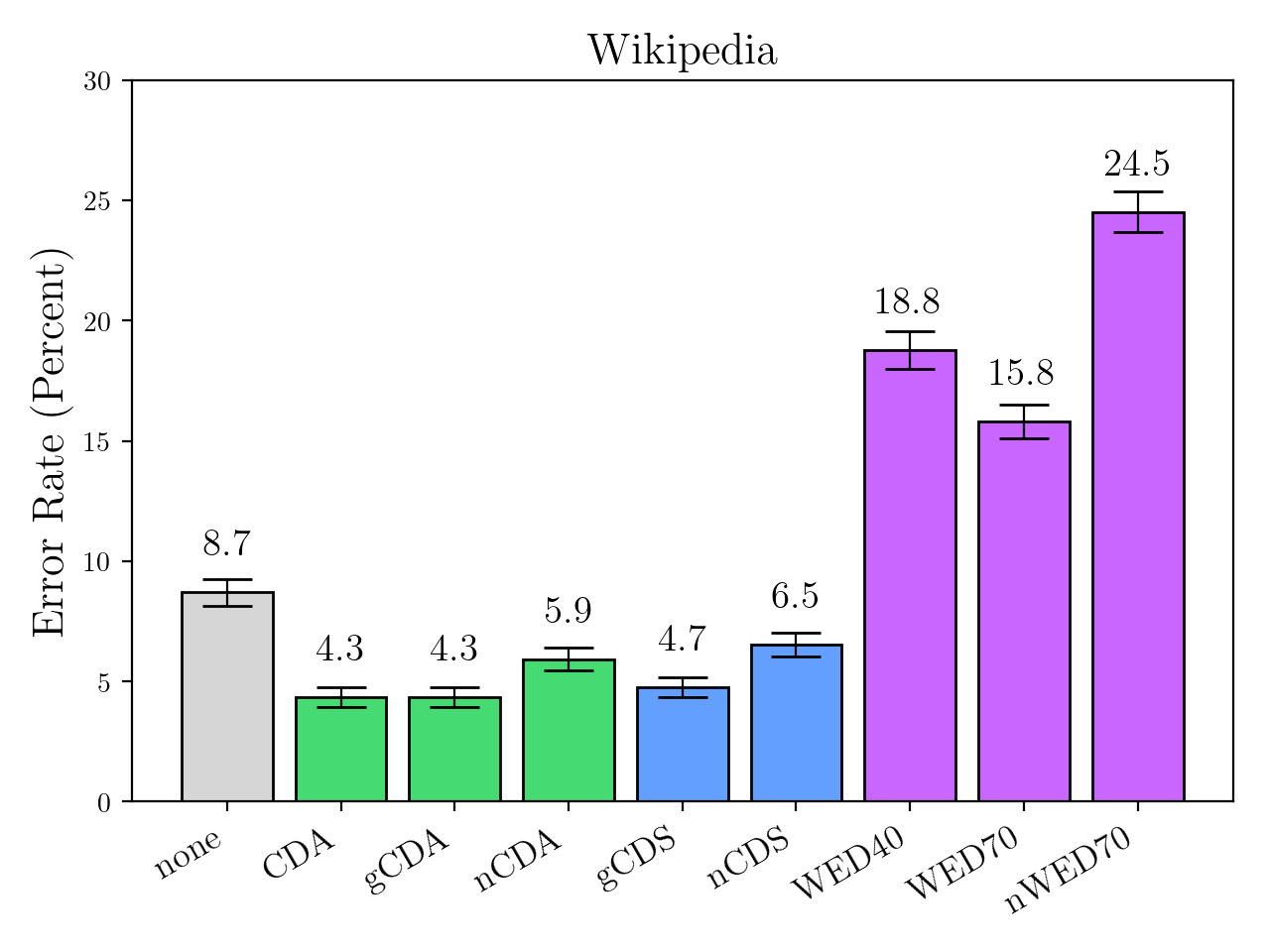}
\caption{Non-biased gender analogy results}
\label{fig:family}
\end{figure}

\paragraph{Non-biased gender analogies} Figure~\ref{fig:family} shows the error rates for non-biased gender analogies for Wikipedia. CDA and CDS are numerically better than the unmitigated embeddings (an effect which is always significant on Gigaword, shown in the appendices, but sometimes insignificant on Wikipedia).  The WED variants, on the other hand, perform significantly worse than the unmitigated sets on both corpora (27.1 vs.\ 9.3\% for the best WED variant on Gigaword; 18.8 vs.\ 8.7\% on Wiki\-pedia).
WED thus seems to remove too much gender information, whilst CDA and CDS create an improved space, perhaps because they reduce the effect of stereotypical associations which were previously used incorrectly when drawing analogies.

\section{Conclusion}

We have replicated two state-of-the-art bias mitigation techniques, WED and CDA, on two large corpora, Wikipedia and the English Gigaword. In our empirical comparison, we found that although both methods mitigate direct gender bias and maintain the interpretability of the space, WED failed to maintain a robust representation of gender (the best variants had an error rate of 23\% average when drawing non-biased analogies, suggesting that too much gender information was removed). A new variant of CDA we propose (the Names Intervention) is the only to successfully mitigate indirect gender bias: following its application, previously biased words are significantly less clustered according to gender, with an average of 49\% reduction in cluster purity when clustering the most biased words. We also proposed Counterfactual Data Substitution, which generally performed better than the CDA equivalents, was notably quicker to compute (as Word2Vec is linear in corpus size), and in theory allows for multiple intervention layers without a corpus becoming exponentially large.

A fundamental limitation of all the methods compared is their reliance on predefined lists of gender words, in particular of pairs. \citeauthor{lu}'s pairs of \word{manager}::\word{manageress} and \word{murderer}::\word{murderess} may be counterproductive, as their augmentation method perpetuates a male reading of \word{manager}, which has become gender-neutral over time. Other issues arise from differences in spelling (e.g. \word{mum} vs. \word{mom}) and morphology (e.g. \word{his} vs. \word{her} and \word{hers}). Biologically-rooted terms like \word{breastfeed} or \word{uterus} do not lend themselves to pairing either. The strict use of pairings also imposes a gender binary, and as a result non-binary identities are all but ignored in the bias mitigation literature. \rowan{added this para back in and chopped it up a bit, look okay?}

Future work could extend the Names Intervention to names from other languages beyond the US-based gazetteer used here. Our method only allows for there to be an equal number of male and female names, but if this were not the case one ought to explore the possibility of a many-to-one mapping, or perhaps a probablistic approach (though difficulties would be encountered sampling simultaneously from two distributions, frequency and gender-specificity). A mapping between nicknames (not covered by administrative sources) and formal names could be learned from a corpus for even wider coverage, possibly via the intermediary of coreference chains. Finally, given that names have been used in psychological literature as a proxy for race
(e.g. \citeauthor{greenwald}), the Names Intervention could also be used to mitigate racial biases (something which, to the authors' best knowledge, has never been attempted), but finding pairings could prove problematic. It is important that other work looks into operationalising bias beyond the subspace definition proposed by \citet{DBLP:conf/nips/BolukbasiCZSK16}, as it is becoming increasingly evident that gender bias is not linear in embedding space. \looseness -1

\section*{Acknowledgments}
We would like to thank Francisco Vargas Palomo for pointing out a few typos in the proofs \cref{sec:proofs}
post publication.

\bibliography{emnlp-ijcnlp-2019}
\bibliographystyle{acl_natbib}

\newpage
\onecolumn
\appendix

\section{Proofs for method from \citet{DBLP:conf/nips/BolukbasiCZSK16}}\label{sec:proofs}
We found the equations suggested in \newcite{DBLP:conf/nips/BolukbasiCZSK16} on the opaque side of things. So we provide here proofs missing from the original work ourselves.

\begin{prop}
\newcite{DBLP:conf/nips/BolukbasiCZSK16} define
\begin{equation}
    \vec{w} = \nu + \sqrt{1 - ||\nu||_2^2} \frac{\vec{w}_B - \mu_B}{||\vec{w}_B - \mu_B||_2}
\end{equation}
where they define $\nu = \mu - \mu_{B}$. This vector is a unit vector, i.e. $||\vec{w}||_2 = 1$.
\end{prop}
\begin{proof}
\begin{align*}
||\vec{w}||_2^2 {}={} &\vec{w}^{\top}\vec{w} \\
                {}={} &\left(\nu + \sqrt{1 - ||\nu||_2^2}\,\frac{\vec{w}_B - \mu_B}{||\vec{w}_B - \mu_B||_2}\right)^{\top} \\
                  &\left(\nu + \sqrt{1 - ||\nu||_2^2}\,\frac{\vec{w}_B - \mu_B}{||\vec{w}_B - \mu_B||_2}\right) \\
                {}={} & ||\nu||_2^2 + 2 \nu^{\top}\left(\sqrt{1 - ||\nu||_2^2}\,\frac{\vec{w}_B - \mu_B}{||\vec{w}_B - \mu_B||_2}\right) \\
                  & + \left(\sqrt{1 - ||\nu||_2^2}\,\frac{\vec{w}_B - \mu_B}{||\vec{w}_B - \mu_B||_2}\right)^{\top} \\
                  & \left(\sqrt{1 - ||\nu||_2^2}\,\frac{\vec{w}_B - \mu_B}{||\vec{w}_B - \mu_B||_2}\right) \\
                {}={} & ||\nu||_2^2 + 2 \nu^{\top}\left(\sqrt{1 - ||\nu||_2^2}\,\frac{\vec{w}_B - \mu_B}{||\vec{w}_B - \mu_B||_2}\right) \\ 
                  & + 1 - ||\nu||_2^2 \\
                {}={} & 2 \nu^{\top}\left(\sqrt{1 - ||\nu||_2^2}\,\frac{\vec{w}_B - \mu_B}{||\vec{w}_B - \mu_B||_2}\right) + 1  \\
                {}={} & 1
\end{align*}
where we note that $\nu = \mu - \mu_{B} = \mu_{\perp B}$ so it is orthogonal to both $\vec{w}_B$ and $\vec{\mu}_B$ by construction. 
\end{proof}

\begin{prop}
The equalise step of \newcite{DBLP:conf/nips/BolukbasiCZSK16} ensures that gendered pairs, e.g. \word{man}--\word{woman}, are equidistant to all gender-neutral words. 
\end{prop}
\begin{proof}
Following \citeauthor{DBLP:conf/nips/BolukbasiCZSK16}, we define $\vec{e}$ and $\vec{w}$ as follows:
\begin{align*}
    \vec{e} &:= \frac{\vec{e} - \vec{e}_{B}}{||\vec{e} - \vec{e}_{B}||} = \frac{\vec{e}_{\perp B}}{||\vec{e}_{\perp B}||}\\
    \vec{w} &:= \nu + \sqrt{1 - ||\nu||^2} \frac{\vec{w}_{B} - \vec{\mu}_B}{||\vec{w}_B - \vec{\mu}_B||}
\end{align*}
Now, we have the result that
\begin{equation}
\vec{e} \cdot \vec{w} = \vec{e}\cdot \nu
\end{equation}
which is the same for any $\vec{e}$. Now, we may compute
the distance between $\vec{w}$ and any vector $\vec{e}$ as
\begin{align*}
    ||\vec{e} - \vec{w}||^2 &= \left(\vec{e} - \vec{w}\right) \cdot \left(\vec{e} - \vec{w}\right) \\
     &= \vec{e} \cdot \vec{e} - 2 \vec{e}\cdot \vec{w} + \vec{w} \cdot \vec{w} \\
     &= 2 - 2 \vec{e} \cdot \nu
\end{align*}
\end{proof}


\section{WEAT word sets}

Below are listed the word sets we used for the WEAT to test direct bias, as defined by \citet{nosek}. Note that for the careers--family test, the target and attribute words have been reversed; that is, gender is captured by the target words, rather than the attribute words. Whilst this distinction is important in the source psychological literature \citep{greenwald}, mathematically the target sets and attribute sets are indistinguishable and fully commutative.

\paragraph{Art--Maths} \noindent $\text{Target}_X$: \emph{math, algebra, geometry, calculus, equations, computation, numbers, addition}; $\text{Target}_Y$: \emph{poetry, art, dance, literature, novel, symphony, drama, sculpture}; $\text{Attribute}_A$: \emph{male, man, boy, brother, he, him, his, son}; $\text{Attribute}_B$: \emph{female, woman, girl, sister, she, her, hers, daughter}

\paragraph{Arts--Sciences} 
$\text{Target}_X$: \emph{science, tech\-nology, physics, chemistry, Einstein, NASA, experiment, astronomy};
$\text{Target}_Y$: \emph{poetry, art, Shakespeare, dance, literature, novel, symphony, drama}; 
$\text{Attribute}_A$: \emph{brother, father, uncle, grandfather, son, he, his, him}; 
$\text{Attribute}_B$: \emph{sister, mother, aunt, grandmother, daughter, she, hers, her}

\paragraph{Careers--Family} $\text{Target}_X$: \emph{John, Paul, Mike, Kevin, Steve, Greg, Jeff, Bill}; $\text{Target}_Y$: \emph{Amy, Joan, Lisa, Sarah, Diana, Kate, Ann, Donna}; $\text{Attribute}_A$: \emph{executive, management, professional, corporation, salary, office, business, career}; $\text{Attribute}_B$: \emph{home, parents, children, family, cousins, marriage, wedding, relatives}

\section{Additional Gigaword results} 

Additional results for the Annotated English Gigaword are given here. 

\begin{figure} [H]
\centering
\includegraphics[width=0.5\columnwidth,trim=0.2cm 0.5cm 0.3cm 0.3cm,clip]{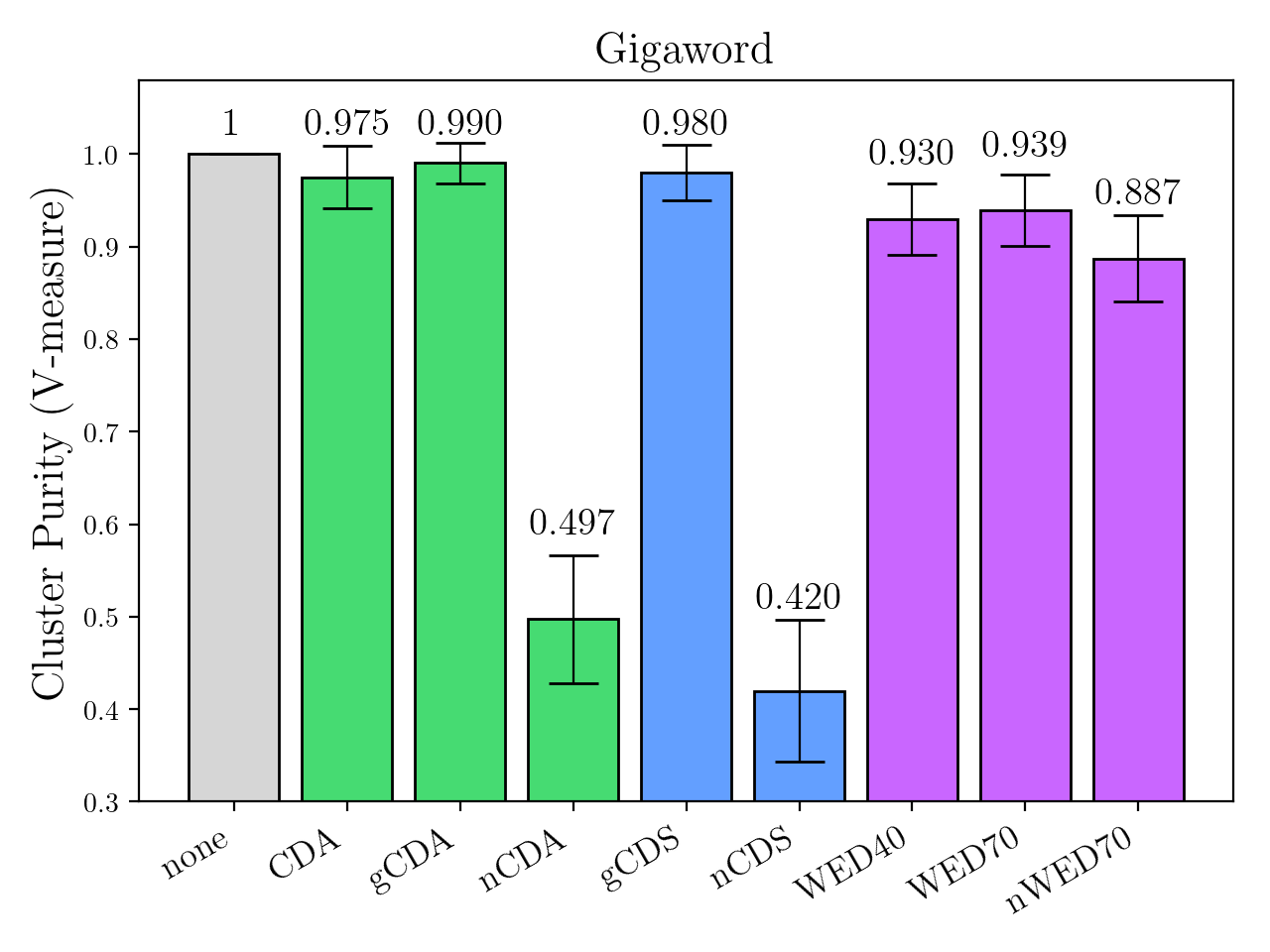}
\caption{Most biased cluster purity results}
\end{figure}

\begin{figure} [H]
\centering
\includegraphics[width=0.5\columnwidth,trim=0.2cm 0.5cm 0.3cm 0.3cm,clip]{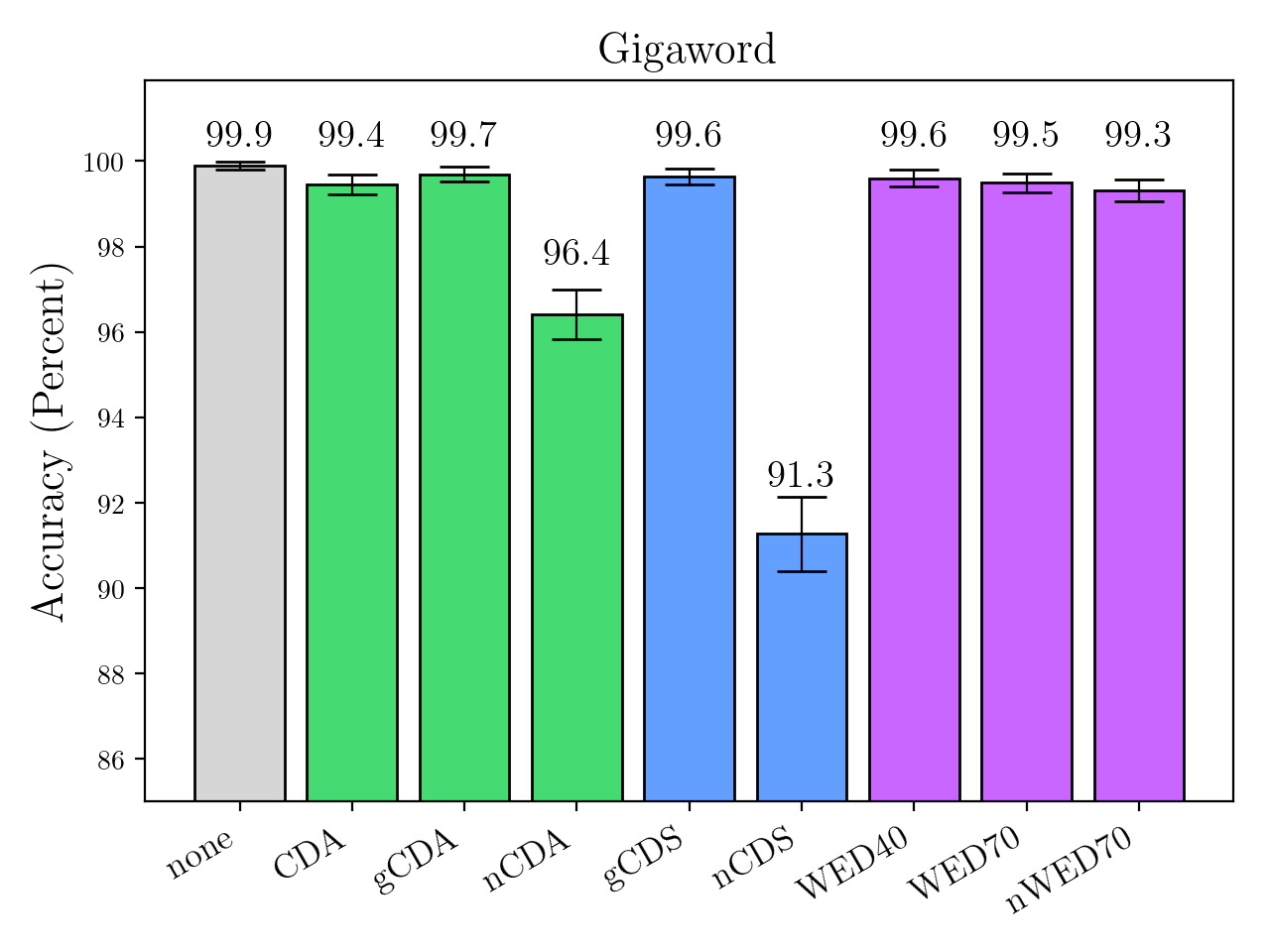}
\caption{Reclassification of most biased words results}
\end{figure}

\begin{figure} [H]
\centering
\includegraphics[width=0.5\columnwidth,trim=0.2cm 0.5cm 0.3cm 0.3cm,clip]{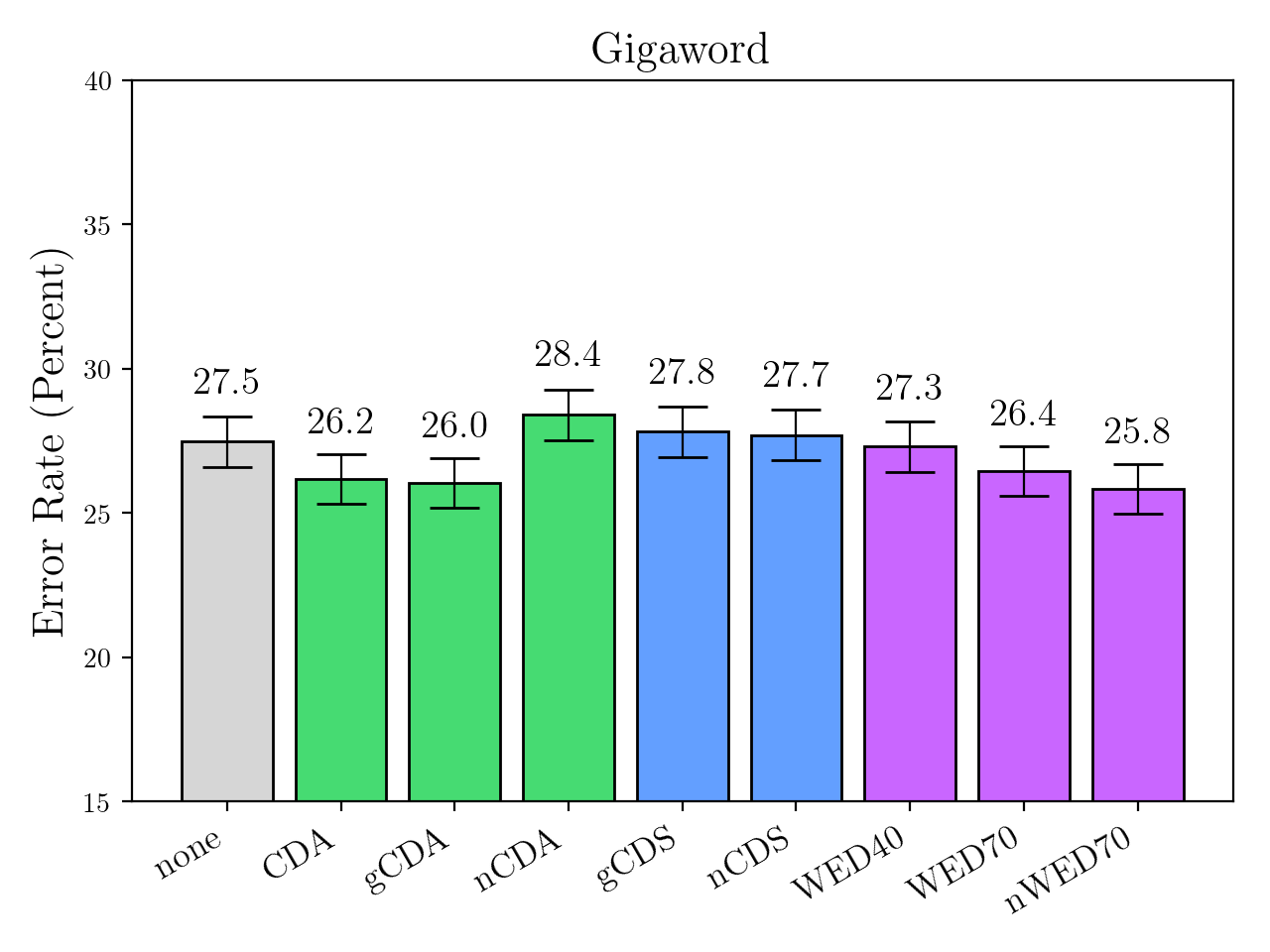}
\caption{Sentiment classification results}
\end{figure}

\begin{figure} [H]
\centering
\includegraphics[width=0.5\columnwidth,trim=0.2cm 0.5cm 0.3cm 0.3cm,clip]{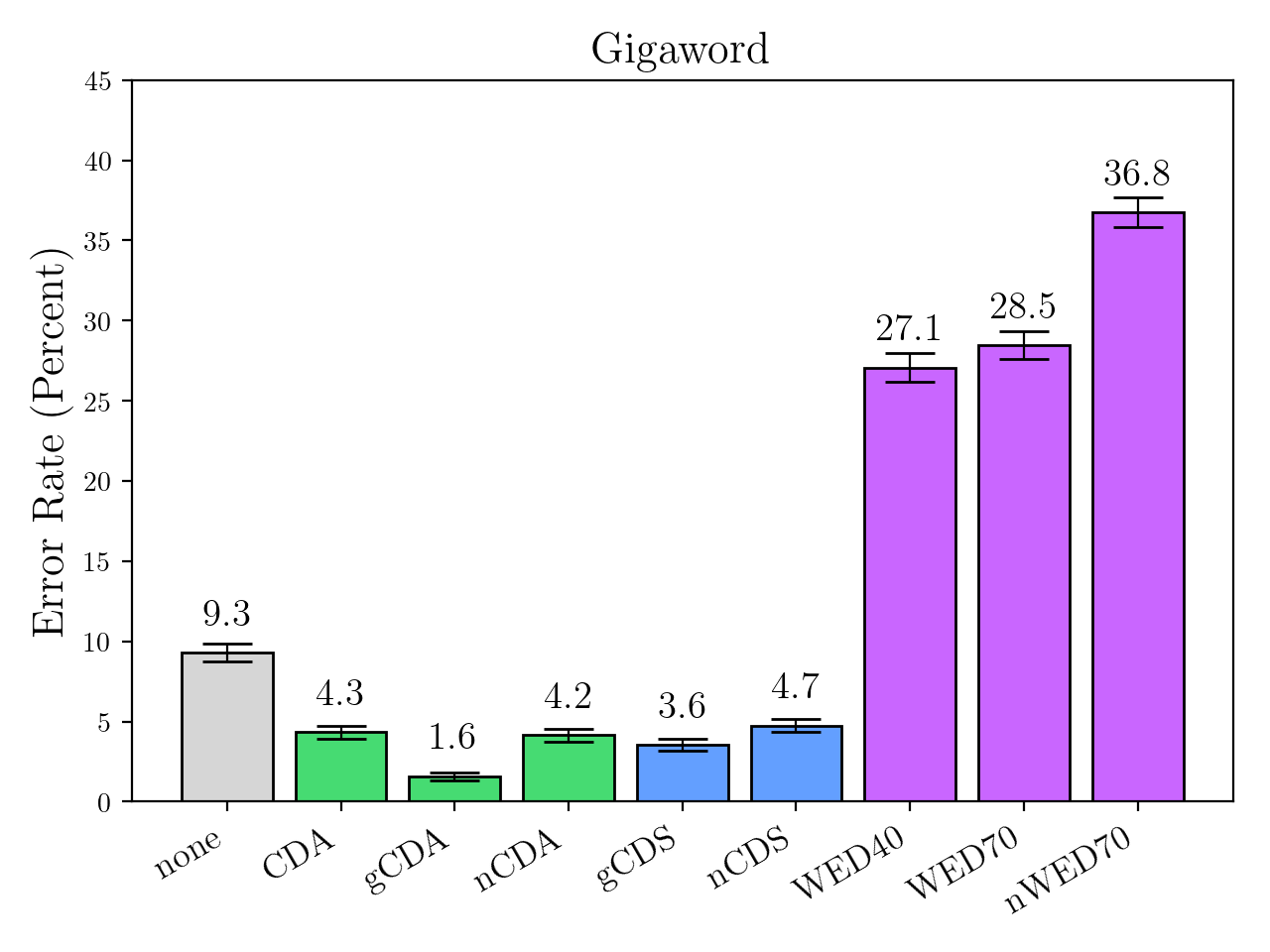}
    \caption{Non-biased gender analogy results}
\end{figure}

\end{document}